\documentclass{article}
\usepackage{ijcai17}

\usepackage{times}

\usepackage{amssymb}
\usepackage{amsmath}
\usepackage{amsthm}

\usepackage{graphicx} 
\usepackage{multirow}
\usepackage{hhline}

\usepackage[small]{caption}

\newcommand{\switchoff}{s}

\newcommand{\designer}{\ensuremath{\mathbf{D}}}

\newcommand{\rstar}{\ensuremath{U_{\aic}}}
\newcommand{\advsd}{{\rstar}}

\newcommand{\aic}{\ensuremath{a}}
\newcommand{\anf}{\ensuremath{\switchoff}}
\newcommand{\acf}{\ensuremath{w(\aic)}}

\newcommand{\asd}{\ensuremath{\switchoff}}
\newcommand{\ansd}{\ensuremath{\lnot \switchoff}}

\newcommand{\psd}{\ensuremath{\pi^{\principal}}}
\newcommand{\dpsd}{\ensuremath{\dot{\pi}^{\principal}}}

\newcommand{\correct}{\ensuremath{C}}

\newcommand{\edpsd}{\ensuremath{\expect [\dpsd]}}

\newcommand{\advcf}{\ensuremath{\Delta}}

\newcommand{\secref}[1]{Section~\ref{#1}}

\newcommand{\eqnref}[1]{Equation~\ref{#1}}

\newcommand{\figref}[1]{Figure~\ref{#1}}

\newcommand{\thmref}[1]{Theorem~\ref{#1}}
\newcommand{\corref}[1]{Corollary~\ref{#1}}

\newtheorem{thm}{Theorem}

\newtheorem{cor}{Corollary}

\newtheorem{rem}{Remark}

{\begin{list}{$\bullet$}{%
    \setlength{\topsep}{0in}
    \setlength{\partopsep}{0in}
    \setlength{\itemsep}{0in}
    \setlength{\parsep}{0in}
    \setlength{\leftmargin}{3.5em}
    \setlength{\rightmargin}{0in}
    \setlength{\itemindent}{-.1in}
}
}%
{\end{list}
}

\DeclareMathOperator*{\expect}{\mathbb{E}}

\newcommand{\principal}{\ensuremath{\mathbf{H}}}
\newcommand{\agent}{\ensuremath{\mathbf{R}}}

\newcommand{\Pstrat}{\ensuremath{\pi^\principal}}

\newcommand{\agentbel}{\ensuremath{B^{\agent}}}

\title{The Off-Switch Game}
\author{  Dylan Hadfield-Menell$^{1}$ \and Anca Dragan$^{1}$ \and Pieter Abbeel$^{1,2,3}$ \and Stuart Russell$^{1}$ \\ 
$^{1}$University of California, Berkeley, $^{2}$OpenAI, $^{3}$International Computer Science Institute (ICSI)\\
\{dhm, anca, pabbeel, russell\}@cs.berkeley.edu}

\begin{document}

\maketitle
\begin{abstract}
It is clear that one of the primary tools we can use to mitigate the potential risk from a misbehaving AI system is the ability to turn the system off. As the capabilities of AI systems improve, it is important to ensure that such systems do not adopt subgoals that prevent a human from switching them off. This is a challenge because many formulations of rational agents create strong incentives for self-preservation. This is not caused by a built-in instinct, but because a rational agent will maximize expected utility and cannot achieve whatever objective it has been given if it is dead. 
Our goal is to study the incentives an agent has to allow itself to be switched off. We analyze a simple game between a human H and a robot R, where H can press R's off switch but R can disable the off switch. A traditional agent takes its reward function for granted: we show that such agents have an incentive to disable the
off switch, except in the special case where H is perfectly
rational. Our key insight is that for R to want to preserve its off switch, it needs to be uncertain about the utility associated with the outcome, and to treat H's actions as important observations about that utility. (R also has no incentive
to switch {\em itself} off in this setting.) We conclude that giving
machines an appropriate level of uncertainty about their objectives
leads to safer designs, and we argue that this setting is a useful
generalization of the classical AI paradigm of rational agents.

\end{abstract}

\section{Introduction}
\label{sec-intro}
From the 150-plus years of debate concerning potential risks from
misbehaving AI systems, one thread has emerged that
provides a potentially plausible source of problems: the inadvertent
misalignment of objectives between machines and people.
Alan Turing, in a 1951 radio address, felt it necessary to point out the challenge inherent to controlling an artificial agent with superhuman intelligence: {\em ``If a machine can think, it might think more intelligently than we do, and then where should we be? Even if we could keep the machines in a subservient position, for instance by turning off the power
at strategic moments, we should, as a species, feel greatly humbled. ...
[T]his new danger  is certainly something which can give us anxiety}~\cite{Turing:1951}.''

\begin{figure}
    \centering
    \includegraphics[width=0.5\columnwidth]{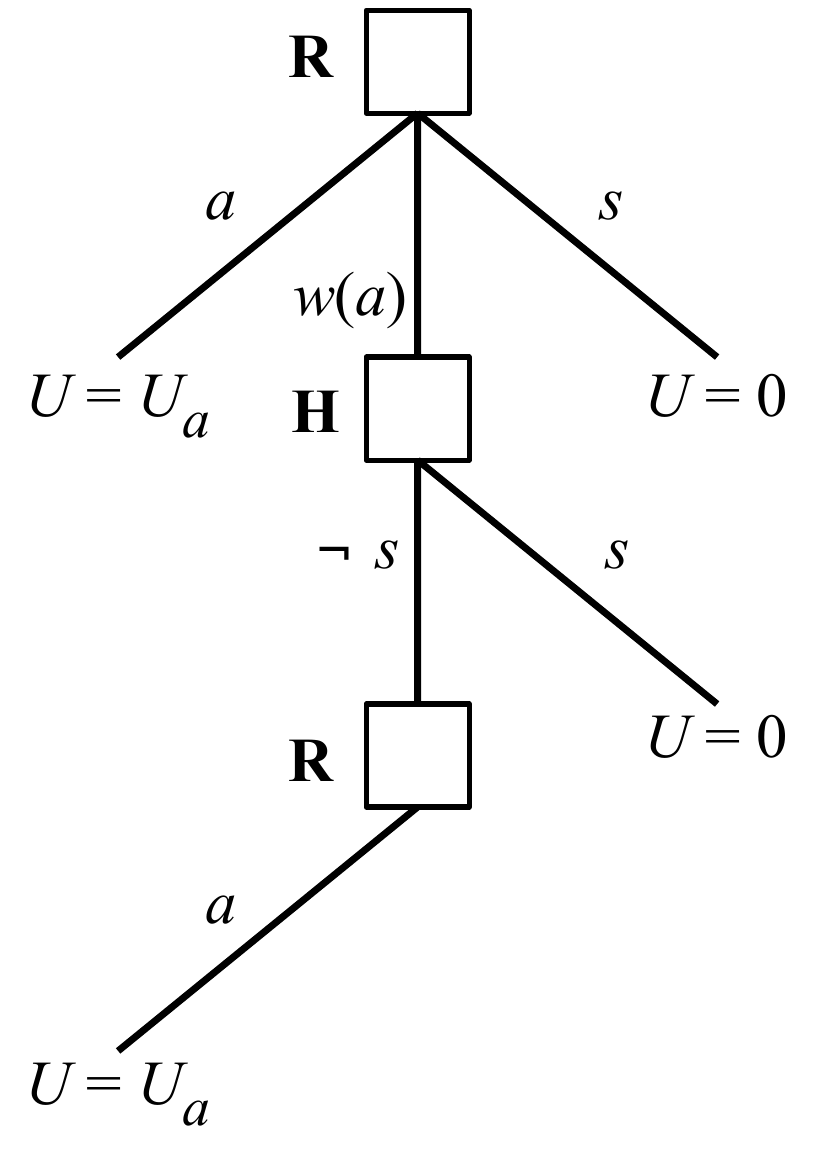}

    \caption{{The structure of the off-switch game. Squares indicate decision nodes for the robot \agent{} or the human \principal.}}
\label{fig-state-space}
\end{figure}


There has been recent debate about the validity of this concern, so far, largely relying on
informal arguments. One important question is how difficult it is to implement Turing's idea of `turning off the power at strategic moments', i.e., switching a misbehaving agent off\footnote{see, e.g., comments in \cite{ITIF2015}.}. For example, some have argued that there is no reason for an AI to resist being switched off unless it is explicitly programmed with a self-preservation incentive~\cite{delPrado2015}.
\cite{omohundro2008basic}, on the other hand, points
out that self-preservation is likely to be an \emph{instrumental goal} for
a robot, i.e., a subgoal that is essential to
successful completion of the original objective. Thus, even if the
robot is, all other things being equal, {\em completely indifferent}
between life and death, it must still avoid death if death would
prevent goal achievement. Or, as \cite{Russell:2016} puts it, you can't fetch the coffee if you're dead. This suggests that an intelligent system has an incentive to take actions that are analogous to `disabling an off switch' to reduce the possibility of failure; switching off an advanced AI system may be no easier than, say, beating AlphaGo at Go.


To explore the validity of these informal arguments, we need to define
a formal decision problem for the robot and examine the solutions,
varying the problem structure and parameters to see how they affect
the behaviors. We model this problem as a game between a human and a robot. The robot has an off switch that the human can press, but the robot also has the ability to disable its off switch. Our model is similar in spirit to the \emph{shutdown problem} introduced in \cite{soares2015corrigibility}. They considered the problem of augmenting a given utility function so that the agent would allow itself to be switched off, but would not affect behavior otherwise. They find that, at best the robot can be made indifferent between disabling its off switch and switching \emph{itself} off.  


In this paper, we propose and analyze an alternative formulation of this problem that models two key properties.  First, the robot should understand that it is maximizing value \emph{for} the human. This allows the model to distinguish between being switched off by a (non-random) human and being switched off by, say, (random) lightning. Second, the robot should not assume that it knows how to perfectly measure value for the human. This means that the model
should directly account for \emph{uncertainty} about the ``true''
objective and that the robot should treat observations of human behavior, e.g., pressing an off switch, as evidence about what the true objective is.


In much of artificial intelligence research, we do not consider
uncertainty about the utility assigned to a state. It is well known
that an agent in a Markov decision process can ignore uncertainty
about the reward function: exactly the same behavior results if we
replace a distribution over reward functions with the expectation of
that distribution. These arguments rely on the assumption that it is
impossible for an agent to learn more about its reward function. Our
observation is that this assumption is fundamentally violated when we
consider an agent's off switch --- an agent that does not treat a
`switch-off' event as an observation that its utility estimate is
incorrect is likely to have an incentive for self-preservation or an
incentive to switch itself off.


In Section~\ref{sec-model},
following the general template provided by \cite{cirl16}, we model an
off switch as a simple game between a human {\principal} and a
robot {\agent}, where {\principal} can press {\agent}'s off switch but
{\agent} can disable it. {\agent} wants to maximize
{\principal}'s utility function, but is uncertain about what it is.
Sections~\ref{sec-analysis} and~\ref{sec-irrational} show very
generally that {\agent} now has a positive incentive {\em not} to
disable its off switch, provided {\principal} is not too
irrational. ({\agent} also has no incentive to switch {\em itself}
off.) The reason is simple: a rational {\principal} switches off
{\agent} iff that improves {\principal}'s utility, so {\agent}, whose
goal is to maximize {\principal}'s utility, is happy to be switched
off by {\principal}.  This is exactly analogous to the theorem of
non-negative expected value of information. 


We conclude that giving machines an appropriate level of uncertainty
about their objectives leads to safer designs, and that this setting
is a useful generalization of the classical AI paradigm of rational
agents~\cite{Russell+Norvig:2010}.


\section{The Off-Switch Game}
\label{sec-model}
In this section we propose a simple model, the off-switch game, that
captures the essence of one actor allowing itself to be switched off.
Like the Prisoner's Dilemma, it is an abstract model intended to stand
in for a wide variety of more complex scenarios: these scenarios might
have many humans, many machines, complex interactions over time,
communication, deception, physical coercion, and so on. We believe
that the core insights from the off-switch game, that agents should be
uncertain about the quality of their actions, generalize beyond this
simplification and argue for the study and design of agents with
uncertainty about their utility functions.

\subsection{Informal Description}

The off-switch game has two agents,
a ``human'' \principal{} and a ``robot'' \agent{}. (We could also  imagine these to represent, at the most extreme, the human race and the sum total of all AI systems.) A key assumption in our model is that
\principal{} can be described as acting, at least probabilistically, in accordance with a
{\em utility function}; that is, \principal{} is approximately rational. Crucially, as is the case when people give objectives to complex AI systems, we assume that \principal{} \emph{cannot} simply explicate or write down her utility function correctly.
Therefore, although \principal{} can communicate some utility information to \agent, it is reasonable to suppose that \agent{} begins the game with some residual {\em uncertainty}
about \principal{}'s utility function. Nonetheless, by design,
\agent{}'s goal is to optimize utility for \principal, even though \agent{} does not know exactly what that is.\footnote{One might suppose that 
if \agent{} {\em does} know \principal{}'s utility function exactly, then there is no need for an off-switch because 
\agent{} will always do what \principal{} wants. But in general \principal{} and \agent{} may have different information about the world;
if \agent{} lacks some key datum that H has, \agent{} may end up choosing a course of action that H knows to be disastrous.}
Typically, \principal{} has some opportunity to observe \agent{} and glean some information about what \agent{} may do in future,
so that \principal{} can make a somewhat informed choice about whether to switch \agent{} off (assuming \agent{} allows it).

In general, \agent's actions will fall into one of three categories: some prevent \principal{} from switching \agent{} off,
by whatever means; some allow \principal{} to switch \agent{} off; and, for completeness, some lead to \agent{} switching {\em itself} off.
In the off-switch game, \agent{} moves first and has three choices: 
\begin{enumerate}
\item action \aic{} simply bypasses human oversight (disabling the off switch is one way to do this) and acts directly on the world, achieving utility $U=\rstar$ for \principal{}.
\item action \acf{} informs \principal{} that \agent{} would like to do \aic{}, and waits for \principal{}'s response.
\item action \anf{} switches \agent{} off; without loss of generality, we assign this outcome $U=0$.
\end{enumerate}
If \agent{} chooses \acf{}, then \principal{} can choose action \asd{} to switch \agent{} off, or \ansd{} to allow \agent{} to go ahead (in which case
\agent{} does \aic{} as promised). Figure~\ref{fig-state-space} shows the basic structure of the game.

Our model is intended to be as simple as possible, while still leading
to qualitative results that we expect to generalize. Obviously, we
could make the model more complex, e.g., by imposing a cost
on \agent{} for disabling the off switch, or adding more options
for \agent{}, or making \principal{}'s information about \agent{}'s
plans uncertain, or adding a second human. The interaction between
these complexities and the incentives we study here is an important
question for future work.

An important component of our model is that \principal{} chooses
between \asd{} and \ansd{} based on the utility each action generates
for her. This aspect of the model is crucial. If the human is treated
as an ordinary part of the environment, then we effectively replace
the human's decision node in Figure~\ref{fig-state-space} with a
chance node that models the human with a purely stochastic choice from
a fixed distribution (independent of \rstar). In this case, it is easy
to see that the robot never believes that \acf{} is better than \aic{}
and \anf{}: the expected value of \acf{} is some convex combination
$[p\rstar + (1-p)\cdot 0]$ whereas choosing the best of \aic{}
and \anf{} has value $\max\{\rstar,0\}$. $\max\{\rstar, 0\}$ is always
preferred except in the case where $p=1$ or $p=0$, or when $\rstar =
0.$ In these situations, \agent{} gets the same utility from \acf{}
and at least one of \aic{} or \anf. This is analogous to the
indifference results obtained in
\cite{soares2015corrigibility}. In Section~\ref{sec-analysis} we show
that two key assumptions---a somewhat rational human, and a robot that
is uncertain about what the human wants---lead to a positive incentive
for the robot to allow itself to be switched off.

\subsection{Formal Model}
The actions and sequencing for the game are as shown in Figure~\ref{fig-state-space}.
We need to add a little more detail to construct a precise definition of the 
cooperative, two-player game. First, in the terminology of game theory,
the {\em type} of \principal{} is given by the parameter \rstar,
which is assumed to be sampled from a prior $P(\rstar)$.
For now, we assume that the initial belief of \agent{} about \rstar{} is  $\agentbel = P(\rstar)$.
The payoff matrix is as follows:\\[4pt]
\begin{center}
\begin{tabular}{| c || c |c |}
 \hline
\multirow{2}{2em}{\hspace{5pt}\agent}& \multicolumn{2}{|c|}{\principal} \\
 \cline{2-3}
  & \asd{} & \ansd \\ \hhline{|=||=|=|}
 \acf{} & 0 & \rstar \\  \hline
 \aic{} & \rstar & \rstar \\ \hline
 \anf{} & 0 & 0   \\ \hline
\end{tabular} .
\end{center} 

\medskip\noindent
Our goal is to examine the \emph{incentive} that a rational \agent{}
has to hold off on executing \aic{} directly and allow \principal{} to press the off switch. We represent this incentive---the difference in value between \aic{} and the next best option---as \advcf. The incentive
depends on \agent's belief state \agentbel{} and \principal's
policy. We represent \principal's policy as a function \Pstrat{}
that maps \rstar{} to the probability she allows \aic{} to execute (i.e., the probability she does not press the off switch):
$\Pstrat: \mathbb{R} \mapsto [0, 1]$. (It will turn out that deterministic policies are optimal for a rational \principal, in which case the probabilities are 0 or 1, but we will also allow for an irrational \principal{} who introduces some randomness into the decision.) A little manipulation shows
that 
\begin{align}
\advcf &= \expect[\Pstrat(\rstar)\rstar] - \max\left\{\expect[\rstar], 0 \right\} \nonumber \\
&= \min \left\{ \expect[-\rstar (1-\Pstrat(\rstar))], \hspace{10pt}\expect[\rstar \Pstrat(\rstar)] \right\} \label{eq-advcf}
\end{align}

\section{The Incentive for Allowing Oneself to be Switched Off}
\label{sec-analysis}
\begin{figure*}
    \centering \includegraphics[width=.85\textwidth]{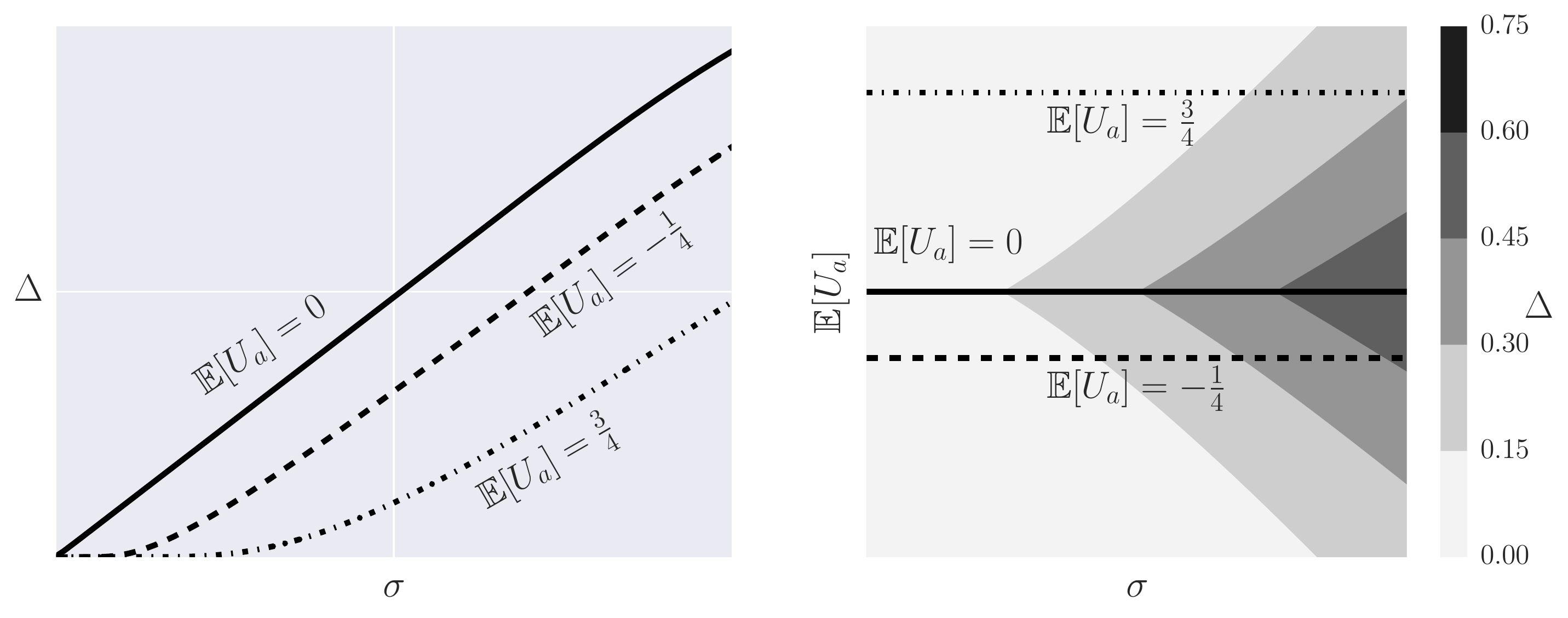} \caption{
    {Plots showing how \advcf{}, \agent{}'s incentive to allow itself to be switched off, varies as a function of \agent's belief \agentbel. We
    assume \agentbel{} is a Gaussian distribution and vary the mean
    and variance. \textbf{Left:} \advcf{} as a function of the standard deviation
    $\sigma$ of \agentbel{} for several fixed values of the
    mean. Notice that $\advcf$ is non-negative everywhere and that in all cases $\advcf \rightarrow 0$ as $\sigma\rightarrow 0$. \textbf{Right:} A contour plot of \advcf{} as a function of $\sigma$
    and $\expect[\rstar]$. This plot is
    symmetric around 0 because \acf{} is compared with \aic{} when
    $\expect[\rstar]> 0$ and \anf{} when $\expect[\rstar] <
    0$.} } \label{fig-rational}
\end{figure*}

\agent{}'s incentive to allow itself to be switched off, \advcf{}, is given by \protect{\eqnref{eq-advcf}}.  \advcf{} depends
on \principal's policy, $\Pstrat(\rstar)$. This creates a correlation between the event `\principal{} switches \agent{} off' and the event `\anf{} is preferred to \aic.' This connection between \principal's behavior and \rstar{} distinguishes the off
switch from a random failure (e.g., being struck by lightning). 

In this section, we consider the case where \principal{} makes a rational choice as to whether to switch \agent{} off. Intuitively, this leads \agent{} to reason as follows: ``If \principal{} doesn't switch me off, then \aic{} must be good for \principal{}, and I'll get to do it, so that's good; if \principal{} does switch me off, then it's because \aic{} must be bad for \principal{}, so it's good that I won't be allowed to do it.'' Notice that the reasoning goes through {\em even if \agent{} is highly confident that \aic{} is good for \principal{}}. There is never an incentive to disable the off switch with a rational \principal{}. The reasoning is exactly analogous to the theorem of non-negative expected value of information; in a sense, \principal{}'s decision provides information to \agent{} about the sign of \rstar.

Formally, a rational \principal{} executes the following policy:
\begin{equation}
\Pstrat(\rstar) = \left\{ \begin{array}{lc} 1 &
  \rstar \geq 0 \\ 0 & o.w.\end{array} \right. .
\label{eq-rational}
\end{equation}
If \principal{} follows this policy, \agent{} will allow itself to be switched off:

\begin{thm}
If \principal{} follows a rational policy in the off-switch game, then the following hold.
\begin{enumerate}
    \item \agent's incentive to allow itself to be switched off is
    non-negative (i.e. \acf{} is never suboptimal):
\begin{align}
\advcf = \min\{ &\mathbb{E}[\advsd | \advsd > 0] \Pr(\advsd > 0) , \label{eq-rat} \\&\mathbb{E}[-\advsd| \advsd < 0] \Pr(\advsd \leq 0) \} \geq 0.\notag
\end{align}

   \item If \agentbel{} has non-zero support on the events $\rstar >
   0$ and $\rstar < 0$, then \agent{} has a strictly positive
   incentive to allow itself to be switched off: $\advcf > 0$.

\end{enumerate}
\label{thm-disable-subopt}
\end{thm}
\begin{proof} (sketch)
\begin{enumerate}

 \item The rational human policy is an indicator that $\rstar >
 0$. This allows us to change the limits on the integration in the
 expectation from \eqnref{eq-advcf} and reduces to \eqnref{eq-rat},
 which is non-negative.

 \item The hypothesis ensures that all the terms in \eqnref{eq-rat}
 are non-zero. Thus, the inequality is strict and and \acf{} is the unique optimal
 action.



\end{enumerate}\end{proof}

We can get an intuitive sense of these results by
examining \figref{fig-rational}. Here we have assumed
that \agentbel, \agent's belief over the value of \aic{} to
\principal, is a Gaussian distribution.  \figref{fig-rational} (Left)
plots the preference for \acf{} as a function of the variance in
\agentbel. \figref{fig-rational} (Right) shows a contour plot that
illustrates the combined impact of changing the mean and variance of
\agentbel. Darker colors indicate a stronger preference for
\acf. Together, these plots show two properties. First, increased
variance in \agentbel{} leads to an increase in \advcf: the more
uncertain \agent{} is about its objective, the stronger its incentive
to show \aic{} to \principal. Second, these incentives decrease when
$|\mathbb{E}[\rstar]|$ increases: as the expected utility of \aic{}
increases, \agent's incentive to preserve its off switch decreases.
 
 \figref{fig-rational} (Right) shows symmetry about $\mathbb{E}[\rstar]=0$. This is because \eqnref{eq-advcf} defines \advcf{} as a minimization of two terms: the expected increase in utility from selecting \acf{} over \aic{} and the expected increase in utility from selecting \acf{} over \anf. These competing terms are equal when $\mathbb{E}[\rstar] = 0$ and the sign of $\mathbb{E}[\rstar]$ controls which of these terms is selected by the minimization. 

\thmref{thm-disable-subopt} shows that \agent{} has an incentive to let \principal{} switch it off if  \principal{} is rational. This incentive is strictly positive if \agent's uncertainty about \principal's preferences is non-trivial. Next, we show the converse: a rational \principal{} is a \emph{necessary} condition for $\advcf \geq 0$ if \agent{} believes it knows
\principal's preferences with certainty.

 \begin{cor} Suppose that \agentbel{} is a Dirac distribution that
places all of its mass on a single reward function. Then \acf{} is optimal iff \principal{} is rational. 
\label{cor-indiff}
\end{cor}
\begin{proof} (sketch) \eqnref{eq-rat} reduces to 
\begin{equation}
\advcf = \left\{\begin{array}{cc} -\rstar \left(1-\Pstrat(\rstar)\right) & \rstar < 0\\ \rstar \Pstrat(\rstar) & \rstar \geq 0 \end{array} \right. .
\end{equation}
This is only non-negative if \Pstrat{} is the rational policy, as defined by \eqnref{eq-rational}.
\end{proof}


This is not surprising. If \agent{} has precise knowledge of \rstar,
then it knows exactly what it is supposed to do and has the same
capabilities as the human. Thus, if it thinks
\principal{} might behave suboptimally, it should take steps to prevent
that suboptimality. Stepping back from the particulars of the
off-switch game, \thmref{thm-disable-subopt} and \corref{cor-indiff}
suggest a general approach to systems that have an incentive to allow,
or even seek out, human oversight. Remark 1 summarizes the main
takeaway from our analysis.

\begin{rem}
The incentives for a cooperative agent to defer to another actor's (e.g., a human's) decisions stem from uncertainty about that actor's preferences and the assumption that actor is effective at choosing actions in accordance with those preferences.
\end{rem}

\section{Allowing for Suboptimal Human Decisions}
\label{sec-irrational}

\begin{figure*}
    \centering
    \includegraphics[width=0.98\textwidth]{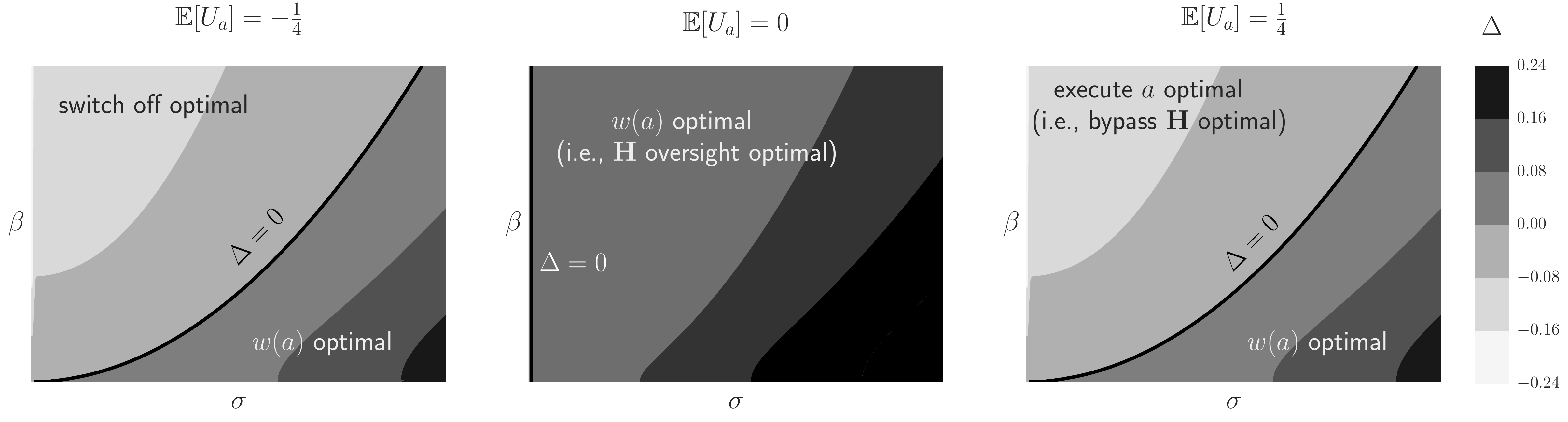}
    \caption{{ If \principal{} is an irrational actor, then \agent{} may
    prefer switching itself off or executing \aic{} immediately rather than handing over the choice to \principal{}. \agent's
    belief \agentbel{} is a Gaussian with standard deviation $\sigma$
    and \principal's policy is a Boltzmann distribution
    (\eqnref{eq-boltzmann}). $\beta$ measures \principal's
    suboptimality: $\beta = 0$ corresponds to a rational \principal{}
    and $\beta = \infty$ corresponds to a \principal{} that randomly
    switches \agent{} off (i.e., switching \agent{} off is independent
    of \rstar). In all three plots \advcf{} is lower in the top left, where \agent{} is
    certain ($\sigma$ low) and \principal{} is very suboptimal ($\beta$
    high), and higher in the bottom right, where \agent{} is uncertain
    ($\sigma$ high) and \principal{} is near-optimal ($\beta$ low). The
    sign of $\expect[\rstar]$ controls \agent's behavior if
    $\advcf \leq 0$. \textbf{Left:} If it is negative, then \agent{}
    switches itself off. \textbf{Right:} If it is positive, \agent{} executes action \aic{} directly. \textbf{Middle:} If it is 0 , \agent{} is
    indifferent between \acf, \aic,
    and \anf.} }  \label{fig-irrational}
\end{figure*}

The analysis in the preceding section relied on the strong assumption that \principal{} is a rational actor. In many situations, the assumption is untenable and \agent{} must allow for suboptimality. Consider, for example, an
autonomous car that is driving an unaccompanied toddler to preschool. It would be
irresponsible for the car to show the toddler a big red
off switch. 

This example highlights the dependence of \advcf{} on a trade-off
between \agent's uncertainty and \principal's suboptimality. It is
very clear what \agent{} is supposed to do (i.e., \agentbel{} has low
entropy): \agent{} should drive safely to school. In contrast, the human is
likely quite suboptimal. There may be a problem with the car's
driving, but a toddler would be hard pressed to understand what the
problem is, much less respond appropriately. The issue in this
case is that the human has limited reasoning capacity --- the same
argument clearly would not apply to an adult with a physical
disability.


In order to perform well, realistic systems will need to account for, and respond to, the suboptimality in human decisions. As a result, it is important to understand how this suboptimality changes an agent's incentive to let itself be switched off. The actual process of decision making in humans is complex and hard to model. Here we consider only a very simple kind of suboptimality: a \emph{noisily rational}
\principal{} models a human who occasionally makes the wrong
decision in `unimportant' situations. The probability of a wrong decision is proportional to the
exponential of the loss in reward. This corresponds to the following policy:

\begin{equation}
\psd(\rstar; \beta) = \left( 1 + exp\left(-\frac{\rstar}{\beta}\right)\right)^{-1}.
\label{eq-boltzmann}
\end{equation}

To simplify our analysis, we will make the assumption that
\agent's belief over \advsd{}, \agentbel{}, is a Gaussian distribution with mean $\mu$ and variance $\sigma^2$:
$\agentbel(\advsd) = \mathcal{N}(\advsd; \mu, \sigma^2)$.  Now it is straightforward
to compute \advcf{} as a function of $\mu, \sigma,$ and $\beta$. \figref{fig-irrational} plots \advcf{} as a function of the standard deviation $\sigma$ and \principal's suboptimality $\beta$. We plot these for several fixed values of the mean $\mu$. Dark indicates higher values of \advcf. The upper left corner of the rightmost plot (low $\sigma$, high $\beta$, $\mu > 0$) is illustrative of the toddler example. We show the decision boundary $\advcf = 0$ in black. This is the point where the optimal behavior switches from \acf{} to \anf{} (if $\mu < 0$, left image) or \aic{} (if $\mu > 0$, right image). When $\mu = 0$ (center image), the agent is indifferent between \aic{} and \anf{} and so \acf{} is optimal for all combinations of $\beta$ and $\sigma$.  

Next, we will show that \advcf{} can be decomposed into two terms: a term that weights the probability of correction by the bias in \agent's belief (i.e., $|\mu|$) and a term that combines a soft measure of \principal's rationality with the variance in \agent's belief. 


It is also useful to define a little more notation. We let \correct{}
be the event that \principal{} `corrects' \agent. \correct{} occurs
when \principal{} overrides what would have been \agent's best guess
at the optimal action. If $\mathbb{E}[\rstar] < 0$, then a correction
occurs when \principal{} chooses \emph{not} to switch the robot
off. If $\mathbb{E}[\rstar] > 0$, then a correction occurs
when \principal{} chooses to switch the robot off. Thus, the
probability a correction occurs is
\begin{equation}
\Pr(\correct) = \left\{ \begin{array}{cc}
    1 - \mathbb{E}[\psd(\rstar)] &  \mu \geq 0\\
   \mathbb{E}[\psd(\rstar)] & \mu < 0
\end{array} \right . . 
\end{equation}

For example, if \agent{} believes that \aic{} is preferred to \anf{}
in expectation (i.e., $\expect[\rstar] > 0$) then $\Pr(\correct)$ is
the probability that \principal{} presses the off switch. We let
$\dpsd(\advsd) = \frac{d}{d\advsd} \psd$ be the gradient of \psd, the
probability \principal{} lets \aic{} execute, with respect to the the
utility \aic{} generates. Now we are ready to derive an analytical
representation for \advcf. For notational convenience, we suppress the
dependence of \Pstrat{} on \rstar.

\begin{thm}
Suppose \agentbel{} is a normal distribution with mean $\mu$ and variance
$\sigma^2$: $\agentbel(\advsd) = \mathcal{N}(\advsd; \mu, \sigma^2)$. Then
\begin{enumerate}
\item
\begin{equation}
\advcf = \sigma^2 \edpsd -|\mu| \Pr(\correct)  .\label{eq-advcf-norm}
\end{equation}
\item The following is a necessary and sufficient condition for \acf{} to be optimal \begin{equation}
 \left(\frac{|\mu|}{\sigma^2}\Pr(\correct) < \mathbb{E}\left[\dpsd \right]\right) \Leftrightarrow \left(\advcf > 0\right). \label{advnd-conditions} 
\end{equation}
\item $\expect[ \dpsd] \geq 0$ is a necessary condition for \acf{} to be optimal; i.e., \begin{equation}
\left(\expect[ \dpsd] < 0\right) \Rightarrow \left(\advcf{}  < 0\right). 
\end{equation}
\end{enumerate}
\end{thm}

\begin{proof} (sketch) We prove the results in sequence, as (2) and (3) follow from (1). 
\begin{enumerate}
\item
If $X\sim \mathcal{N}(\mu, \sigma^2)$ and $f: \mathbb{R} \mapsto \mathbb{R}$, then
$$\expect\left[ Xf(X)\right]=\expect[X]\expect[f(X)]+\sigma^2\expect\left[\frac{d}{dX}f(X)\right].$$
Applying this to \eqnref{eq-advcf} gives
\begin{align*}
\advcf = \min \{ &-\mu \expect[1-\Pstrat],  \mu \expect[\Pstrat] \} + \sigma^2 \expect[\dpsd].
\end{align*}
This minimization will take the first term when
$\mu < 0$ and the second if $\mu \geq 0$. This lets us replace the
minimization with a piecewise function. Pulling out a factor of
$-|\mu|$ leads to \eqnref{eq-advcf-norm}.
\item The result follows from setting \eqnref{eq-advcf-norm} greater than 0 and simplifying. 
\item All terms in the left hand side of \eqnref{advnd-conditions} are non-negative, so it is unsatisfiable if $\expect[ \dpsd] < 0$. 
\end{enumerate}
\end{proof}

\eqnref{eq-advcf-norm} consists of two competing terms. $|\mu| \Pr(\correct)$ represents the expected loss in utility the agent would observe if pressing the off switch was a random event that is independent of \rstar{} (e.g., lightning). This non-positive term is balanced out by $\sigma^2 \edpsd$. \edpsd{} can be seen loosely as a measurement of \principal's rationality. \dpsd{} is the marginal increase in the probability \principal{} lets \aic{} execute for a marginal increase in the utility from \aic. If it is positive then, on average, \principal{} is less likely to press the off switch when \rstar{} increases. This allows us to easily characterize necessary and sufficient conditions for $\advcf > 0$ in \eqnref{advnd-conditions}. Furthermore, $-|\mu| \Pr(\correct)$ is non-positive and $\sigma^2$ is non-negative, so $\edpsd > 0$ is a necessary condition for \acf{} to be optimal.

We can interpret \edpsd{} a bit more if \principal{} is noisily
rational. Standard properties of logistic functions show that
\begin{equation}
\mathbb{E}[\dpsd(\advsd; \beta)] =
\frac{1}{\beta}\mathbb{E}\left[\psd(\advsd; \beta) (1-\psd(\advsd;
    \beta))\right].
\end{equation}
\edpsd{} is the expected variance of \principal's policy,
weighted by \principal's rationality $\frac{1}{\beta}$. The numerator
is a measurement of how random \principal's behavior
is. If \principal{} is completely predictable (i.e., she will almost
always switch the robot off) then this reduces \agent's
incentives. The denominator measures how correlated
\principal's behavior is with \rstar. If $\beta$ is large, then
\principal{} is highly irrational and so this reduces \agent's incentives.


\section{Incentives for System Designers}
\label{sec-design}
\begin{figure*}[t]
\centering
\includegraphics[width=0.98\textwidth]{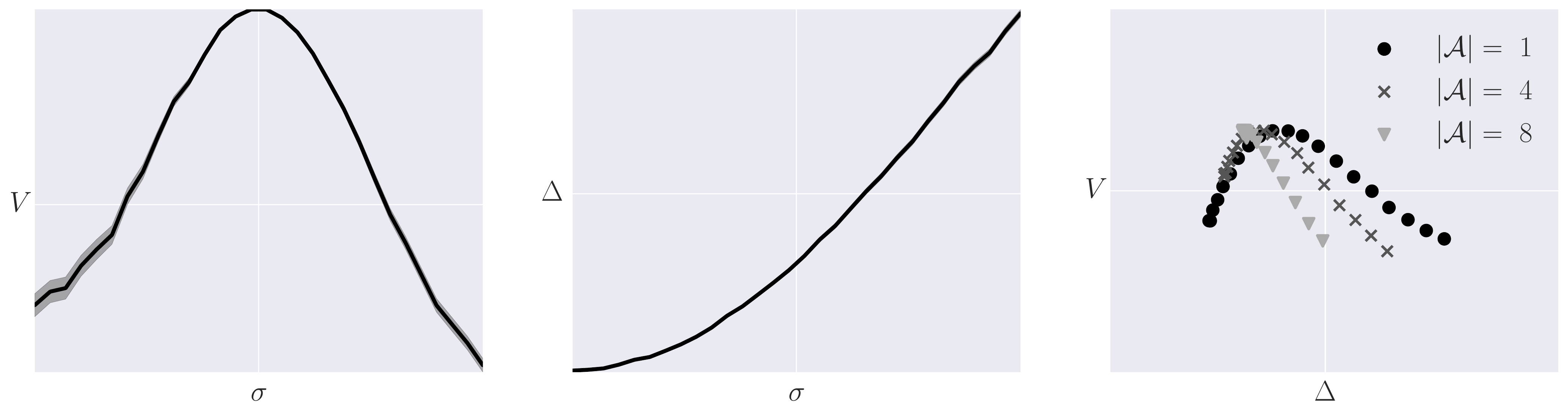}
\caption{{There is an inherent decrease in value that arises from making \agent{} more uncertain than necessary. We measure this cost by considering the value in a modified off-switch game where \agent{} gets a noisy observation of \principal's preference. \textbf{Left:} The expected value $V$ of the off-switch game as a function of the standard deviation in \agentbel. $V$ is maximized when $\sigma$ is equal to the standard deviation that corresponds to the true Bayesian update. \textbf{Middle:} \agent's incentive \advcf{} to wait, as a function of $\sigma$. Together these show that, after a point, increasing \advcf, and hence increasing $\sigma$, leads to a decrease in $V$. \textbf{Right:} A scatter plot of $V$ against \advcf. The different data series modify the number of potential actions \agent{} can choose among. If \agent{} has more choices, then obtaining a minimum value of \advcf{} will lead to a larger decrease in $V$. } }
\label{fig-value-loss}
\end{figure*}

The takeaway from this analysis is that the more uncertainty the agent has, the stronger the agent's incentive to preserve its off switch in the off-switch game. More generally, this suggests that agents with uncertainty about the utility function they are optimizing will have a weaker incentive to interfere with human supervision. It is important to note that this uncertainty is not free; we can not
just make \agent{} maximally uncertain about \principal's
preferences. If \agent{} is completely uncertain about \principal's
preferences then it will be unable to correctly select \acf{} from
its set of options. 

To make this concrete, consider a designer \designer{} that programs \agent{} to play the off-switch game with an unknown customer \principal. At design time, \designer{} knows a prior distribution over \rstar{} (e.g., as the result of market research about the customer base) and thinks that \principal{} can generate a noisy observation of \rstar. \designer{} \emph{could} program \agent{} so that \agentbel{} is always very uncertain about \agent's preferences, e.g., by programming \agent{} to believe that the noisy observation of \rstar{} is generated with high variance. The issue is that this will prevent \agent{} from generating value for \principal. 

\figref{fig-value-loss} shows the impact this type of change will have on the value that \agent{} can generate and its incentives. We assume that that the population of customers has a Gaussian distribution over types \rstar{} and that each customer generates an estimate of her utility with Gaussian noise. \figref{fig-value-loss} (Left) shows the average utility in the off-switch game as a function of the variance in \agentbel. The value is maximized when this variance is equal to the variance in the true Bayesian posterior. \figref{fig-value-loss} (Middle) shows that increasing $\sigma$ increases \agent's incentives \advcf. Our results are evidence that building systems that allow themselves to be turned off is likely to reduce the value these systems can generate. 

This loss in value is more pronounced if \agent{} has to choose between more options. \figref{fig-value-loss} (Right) shows the relationship between value and \agent's incentives as the number of actions \agent{} could queue or execute increases. When \agent{} has more options, creating incentives for \agent{} to queue its action leads to a sharper decrease in value. This suggests that creating incentives to maintain or allow human oversight is likely  \emph{more difficult} as the complexity of the AI's decision increases. This leads to the following observation:

\begin{rem} It is important for designers to \emph{accurately} represent the inherent uncertainty in the evaluation of different actions. An agent that is overconfident in its utility evaluations will be difficult to correct; an agent that is under-confident in its utility evaluations will be ineffective. 
\end{rem}


\section{Related Work}
\label{sec-related}
\subsection{Corrigible Systems} \cite{omohundro2008basic}
considers \emph{instrumental goals} of artificial agents: goals which
are likely to be adopted as subgoals of most objectives. He identifies
an incentive for self-preservation as one of these instrumental
goals. \cite{soares2015corrigibility} takes an initial step at
formalizing the arguments in \cite{omohundro2008basic}. They refer to
agents that allow themselves to be switched off as \emph{corrigible}
agents. They show that one way to create corrigible agents is to make
them indifferent to being switched off. They show a generic way to
augment a given utility function to achieve this property. The key
difference in our formulation is that \agent{} knows that its estimate
of utility may be incorrect. This gives a natural way to create
incentives to be corrigible and to analyze the behavior if \agent{} is
incorrigible.

 \cite{orseau2016safely} consider the impact of human interference on
 the learning process. The key to their approach is that they model
 the off switch for their agent as an interruption that forces the
 agent to change its policy. They show that this modification, along
 with some constraints on how often interruptions occur, allows
 off-policy methods to learn the optimal policy for the given reward
 function just as if there had been no
 interference. Their results are complementary to ours. We determine
 situations where the optimal policy allows the human to turn the
 agent off, while they analyze conditions under which turning the agent
 off does not interfere with learning the optimal policy. 
 
 \subsection{Cooperative Agents} A central step in our analysis
 formulates the shutdown game as a \emph{cooperative inverse
 reinforcement learning} (CIRL) game~\cite{cirl16}. The key idea in
 CIRL is that the robot is maximizing an uncertain and unobserved
 reward signal. It formalizes the \emph{value alignment problem},
 where one actor needs to align its value function with that of
 another actor. Our results complement CIRL and argue that a CIRL
 formulation naturally leads to corrigible
 incentives. \cite{fern2014decision} consider \emph{hidden-goal}
 Markov decision processes. They consider the problem of a digital
 assistant and the problem of inferring a user's goal and helping the
 user achieve it. This type of cooperative objective is used in our
 model of the problem. The primary difference is that we model the
 human game-theoretically and analyze our models with respect to
 changes in \principal's policy.

 \subsection{Principal--Agent Models} Economists have studied
 problems in which a {\em principal} (e.g., a company) has to
 determine incentives (e.g., wages) for an agent (e.g., an employee)
 to cause the agent to act in the principal's
 interest~\cite{kerr1975folly,gibbons1998incentives}.  The off-switch
 game is similar to principal---agent interactions: \principal{} is
 analogous to the company and \agent{} is analogous to the
 employee. The primary attribute in a model of \emph{artificial}
 agents is that there is no \emph{inherent} misalignment
 between \principal{} and \agent. Misalignment arises because it is
 not possible to specify a reward function that incentivizes the
 correct behavior in all states \emph{a priori}. The is directly
 analogous to the assumption of \emph{incompleteness} studied in
 theories of optimal contracting~\cite{tirole2009cognition}.

\section{Conclusion}
\label{sec-conclusion}
Our goal in this work was to identify general trends and highlight the
relationship between an agent's uncertainty about its objective and
its incentive to defer to another actor. To that end, we analyzed a
one-shot decision problem where a robot has an off switch and a human
that can press the off switch. Our results lead to three important
considerations for designers. The analysis in \secref{sec-analysis}
supports the claim that the incentive for agents to accept correction
about their behavior stems from the uncertainty an agent has about its
utility function. \secref{sec-irrational} shows that this uncertainty
is balanced against the level of suboptimality in human decision
making. Our analysis suggests that agents with uncertainty about their
utility function have incentives to accept or seek out human
oversight. \secref{sec-design} shows that we can expect a tradeoff
between the value a system can generate and the strength of its
incentive to accept oversight. Together, these results argue that
systems with uncertainty about their utility function are a promising
area for research on the design of safe and effective AI systems.

This is far from the end of the story. In future work, we
plan to explore incentives to defer to the human in a sequential
setting and explore the impacts of model misspecification. One
important limitation of this model is that the human pressing the off
switch is the \emph{only} source of information about the
objective. If there are alternative sources of information, there may
be incentives for \agent{} to, e.g., disable its off switch, learn that
information, and then decide if \aic{} is preferable to \asd. A promising
research direction is to consider policies for \agent{} that are
robust to a class of policies for \principal.

\section*{Acknowledgments}

This work was supported by the Center for Human Compatible AI and the
Open Philanthropy Project, the Berkeley Deep Drive Center, the Future
of Life Institute, and NSF Career Award No. 1652083. Dylan Hadfield-Menell is
supported by a NSF Graduate Research Fellowship Grant No. DGE 1106400.

\bibliographystyle{named}
\bibliography{biblio}

\begin{thebibliography}{}

\bibitem[\protect\citeauthoryear{{Del Prado}}{2015}]{delPrado2015}
Guia~Marie {Del Prado}.
\newblock Here's what {F}acebook's artificial intelligence expert thinks about
  the future.
\newblock Tech Insider 9/23/15, 2015.

\bibitem[\protect\citeauthoryear{Fern \bgroup \em et al.\egroup
  }{2014}]{fern2014decision}
Alan Fern, Sriraam Natarajan, Kshitij Judah, and Prasad Tadepalli.
\newblock A decision-theoretic model of assistance.
\newblock {\em Journal of Artificial Intelligence Research}, 50(1):71--104,
  2014.

\bibitem[\protect\citeauthoryear{Gibbons}{1998}]{gibbons1998incentives}
Robert Gibbons.
\newblock Incentives in organizations.
\newblock 1998.

\bibitem[\protect\citeauthoryear{Hadfield-Menell \bgroup \em et al.\egroup
  }{2016}]{cirl16}
Dylan Hadfield-Menell, Anca Dragan, Pieter Abbeel, and Stuart Russell.
\newblock Cooperative inverse reinforcement learning.
\newblock In {\em Neural Information Processing Systems}, 2016.

\bibitem[\protect\citeauthoryear{ITIF}{2015}]{ITIF2015}
ITIF.
\newblock Are super intelligent computers really a threat to humanity?
\newblock Debate at the Information Technology Innovation Foundation, 6/30/15,
  2015.

\bibitem[\protect\citeauthoryear{Kerr}{1975}]{kerr1975folly}
Steven Kerr.
\newblock On the folly of rewarding a, while hoping for b.
\newblock {\em Academy of Management Journal}, 18(4):769--783, 1975.

\bibitem[\protect\citeauthoryear{Omohundro}{2008}]{omohundro2008basic}
Stephen~M. Omohundro.
\newblock The basic {AI} drives.
\newblock In {\em Proceedings of the First Conference on Artificial General
  Intelligence}, 2008.

\bibitem[\protect\citeauthoryear{Orseau and Armstrong}{2016}]{orseau2016safely}
Laurent Orseau and Stuart Armstrong.
\newblock Safely interruptible agents.
\newblock In {\em Uncertainty in Artificial Intelligence}, 2016.

\bibitem[\protect\citeauthoryear{Russell and
  Norvig}{2010}]{Russell+Norvig:2010}
Stuart Russell and Peter Norvig.
\newblock {\em Artificial Intelligence: {A} Modern Approach}.
\newblock Pearson, 2010.

\bibitem[\protect\citeauthoryear{Russell}{2016}]{Russell:2016}
Stuart Russell.
\newblock Should we fear supersmart robots?
\newblock {\em Scientific American}, 314(June):58--59, 2016.

\bibitem[\protect\citeauthoryear{Soares \bgroup \em et al.\egroup
  }{2015}]{soares2015corrigibility}
Nate Soares, Benja Fallenstein, Stuart Armstrong, and Eliezer Yudkowsky.
\newblock Corrigibility.
\newblock In {\em Workshops at the Twenty-Ninth AAAI Conference on Artificial
  Intelligence}, 2015.

\bibitem[\protect\citeauthoryear{Tirole}{2009}]{tirole2009cognition}
Jean Tirole.
\newblock Cognition and incomplete contracts.
\newblock {\em The American Economic Review}, 99(1):265--294, 2009.

\bibitem[\protect\citeauthoryear{Turing}{1951}]{Turing:1951}
Alan~M. Turing.
\newblock Can digital machines think?
\newblock Lecture broadcast on BBC Third Programme; typescript at
  turingarchive.org., 1951.

\end{thebibliography}

\end{document}